\documentclass{article} 
\usepackage{iclr2025_conference,times}

\usepackage{amsmath,amsfonts,bm}

\def\eqref#1{equation~\ref{#1}}

\def\1{\bm{1}}

\DeclareMathAlphabet{\mathsfit}{\encodingdefault}{\sfdefault}{m}{sl}
\SetMathAlphabet{\mathsfit}{bold}{\encodingdefault}{\sfdefault}{bx}{n}

\usepackage{graphicx} 
\usepackage{amsmath, amssymb}
\usepackage{amsthm}
\theoremstyle{plain}
\newtheorem*{theorem*}{Theorem}
\usepackage{hyperref}
\usepackage{url}
\usepackage{algorithm}
\usepackage{algorithmic}
\usepackage{graphicx}
\usepackage{tikz}
\usepackage{enumitem}
\usepackage[utf8]{inputenc}
\usepackage{newunicodechar}
\DeclareUnicodeCharacter{202F}{~}
\usepackage{enumitem}
\title{BAR Conjecture: the Feasibility of Inference Budget-Constrained LLM Services with Authenticity and Reasoning}

\author{Jinan Zhou\thanks{Equal contribution.} \\
Nutanix \\
\texttt{jinan.zhou@nutanix.com} \\
\And Rajat Ghosh\footnotemark[1] \\
Nutanix \\
\texttt{rajat.ghosh@nutanix.com} \\
\AND Vaishnavi Bhargava \\
Nutanix \\
\texttt{vaishnavi.bhargava@nutanix.com} \\
\And Debojyoti Dutta \\
Nutanix \\
\texttt{debojyoti.dutta@nutanix.com} \\
\And Aryan Singhal \\
Nutanix \\
\texttt{aryan.singhal@nutanix.com} \\
}

\iclrfinalcopy 
\begin{document}

\maketitle

\begin{abstract}

When designing LLM services, practitioners care about three key properties: inference‑time budget, factual authenticity, and reasoning capacity. However, our analysis shows that no model can simultaneously optimize for all three. We formally prove this trade‑off and propose a principled framework named The BAR Theorem for LLM‑application design.

\end{abstract}

\section{Introduction}

Large language models (LLMs) with the Transformer architecture \citep{vaswani2017attention}  are pre-trained with a massive number of tokens to be auto-regressive next token predictors, \( \prod_{t=1}^{T} P(x_t \mid x_{<t}) \). An LLM parameter space, \( \theta \), is split in the following way: \( \theta: (\theta_{\text{PE}}, \theta_{\text{TE}},  \{ \theta^{(l)}_{\text{ATTN}}, \theta^{(l)}_{\text{FF}} \}_{l=0}^{L-1}, \theta_{\text{OUTPUT}}) \). LLMs learn instruction-following during post-training phase using RL-based techniques such as RLHF \citep{ouyang2022traininglanguagemodelsfollow}, GRPO \citep{deepseekai2025deepseekr1incentivizingreasoningcapability}; preference optimization techniques such as DPO \citep{rafailov2024directpreferenceoptimizationlanguage}, KTO \citep{ethayarajh2024ktomodelalignmentprospect}; and supervised techniques such as LoRA \citep{hu2021loralowrankadaptationlarge}. The instruction-following, \( \prod_{t=1}^{T} P(y_t \mid y_{<t}, x; \theta)\), means the model, $\theta$, learning to steer its generation trajectories  $\prod_{t=1}^{T} P(y_t \mid y_{<t})$  based on the instruction prompt, $x$. 

Two key benchmarking attributes for a large language model (LLM) are: first, \textbf{authenticity} or \textbf{faithfulness} — its ability to remain grounded in factual information \citep{lin2022truthfulqameasuringmodelsmimic,yang-etal-2018-hotpotqa}, avoid hallucination \citep{li2023haluevallargescalehallucinationevaluation,joshi2017triviaqalargescaledistantly}, and resist sycophancy \citep{sharma2025understandingsycophancylanguagemodels}; second, \textbf{reasoning} — its capability to perform logical tasks such as mathematical reasoning \citep{cobbe2021trainingverifierssolvemath,hendrycks2021math,muxin2025bigmath,fang2024mathodyssey,glazer2024frontiermath,gao2024omnimath}, logical reasoning \citep{wang2024multilogieval,yao2024logicbench,wang2024mme}, commonsense reasoning \citep{talmor2019commonsenseqa}, STEM question answering \citep{rein2023gpqagraduatelevelgoogleproofqa}, scientific research \citep{gottweis2025aicoscientist}, cross-domain problem solving \citep{wang2024hle}, and coding \citep{patel2024aimeaioptimizationmultiple,zheng2023swebench,aider2024polyglot,li2023livecodebench,novikov2025alphaevolvecodingagentscientific}. While authenticity and reasoning are qualitative measures evaluating the capabilities of LLMs, inference overhead is a \textbf{budget} metric that measures the operational cost of deploying LLMs. Inference for large language models (LLMs) is primarily a memory-bound process due to several factors: these models typically contain billions or even trillions of parameters, requiring constant memory accesses during inference~\citep{brown2020language, chowdhery2022palm}. The arithmetic operations per memory access are relatively simple and sparse, causing inference performance to become constrained primarily by memory bandwidth rather than compute throughput~\citep{pope2022efficiently, dao2022flashattention}. Additionally, transformer-based LLMs frequently access attention-related caches (KV-caches), increasing memory traffic significantly~\citep{vaswani2017attention, shazeer2019fast}. Large model sizes often approach or exceed GPU memory capacity, necessitating frequent memory transfers or sophisticated memory management strategies ~\citep{kwon2023efficientmemorymanagementlarge}, further exacerbating memory bottlenecks~\citep{rajbhandari2020zero, shoeybi2019megatronlm}. Consequently, memory operations dominate, making inference predominantly memory-bound.

\subsection{Empirical Observation on LLM App Design Trade-offs}

We have empirically observed that, for any given LLM application, satisfying all three key design criteria—budget, authenticity, and reasoning—is practically infeasible. Typically, only two of these criteria can be simultaneously met, leading to inherent trade-offs. Below we discuss these trade-offs explicitly:

\begin{itemize}
    \item \textbf{Budget vs. Authenticity:} Enhancing authenticity often necessitates retrieval-augmented generation increasing latency through additional steps such as document retrieval and reranking. Techniques like model ensembles or chain-of-verification further extend latency due to multiple inference steps. Consequently, many real-time applications (e.g., voice assistants) choose to forgo retrieval augmentation, accepting a higher risk of hallucinations to maintain budget constraints. 
    \item \textbf{Authenticity vs. Reasoning:} Improving reasoning capabilities frequently exacerbates the risk of hallucinations, especially when models attempt inference beyond their training data. Methods such as chain-of-thought and zero-shot reasoning amplify rationalized hallucinations if models lack sufficient grounding in specific domains. For instance, GPT-4 can accurately follow logical steps to solve math problems but may hallucinate incorrect final answers if grounding in units or boundary conditions is missing.
    
    \item \textbf{Reasoning vs. Budget:} Achieving complex reasoning typically requires additional computational resources, such as multi-hop retrieval, deeper context windows, and frameworks like Chain-of-Thought \citep{wei2022chainofthought} prompting or scratchpad reasoning \citep{nye2021scratchpad}. Advanced models (e.g., GPT-4-Turbo, Claude 3 Opus, Gemini 2.5) often employ longer contexts or iterative inference strategies, directly increasing latency and computational overhead. For example, frameworks such as Tree-of-Thoughts \citep{yao2023treeofthought}, ReAct \citep{yao2023react}, Least-to-Most prompting \citep{zhou2022leasttomost}, Self-Consistency \citep{wang2022selfconsistency}, and analogical reasoning \citep{yao2023analogical} explicitly trade increased latency for improved reasoning performance.
\end{itemize}

\subsection{Research Question}

The key research question explored in this paper is whether we can design an LLM app that can simultaneously satisfy a fixed inference budget, concurrently maintaining both authenticity and reasoning. 

\section{Methodology}

\subsection{Preliminaries}
Let an LLM be a stochastic algorithm \(M\) with parameters
\(\theta\in\Theta\).  \(\mathcal{L}_{\mathrm{budget}}(\theta)\), 
\(\mathcal{L}_{\mathrm{auth}}(\theta)\), and \(\mathcal{L}_{\mathrm{reason}}(\theta)\) denote loss functions for \emph{inference budget}, \emph{authenticity}, and \emph{reasoning}, respectively. The inference budget is measured in the end-to-end wall‑clock time and is the sum of
\emph{compute latency} (forward passes and intermediate
decoding steps) and \emph{retrieval latency}
(e.g.\ external API or vector‐DB look‑ups).

\vspace{0.5em}
\noindent\textbf{Premise A1 (Reasoning needs inference-time budget).}  Given input length $n$, it have been shown that constant-depth transformers with finite precision poly(n) embedding size can only solve problems in
$TC^{0}$ without CoT \cite{merrill2022saturated, strobl2023averagehard, merrill2025logdepth}. CoT and other inference-time intermediate steps enhance reasoning by forcing the model to emit an explicit sequence of intermediate tokens that decomposes the original problem; these “scratch‑pad” tokens effectively deepen the computation, letting later self‑attention layers attend to and refine partial conclusions, which empirically improves multi‑step logical and arithmetic accuracy without modifying model weights \citep{li2024chainthoughtempowerstransformers, nye2021workscratchpadsintermediatecomputation}. For a family of length‑parameterised reasoning tasks
\(\{T_n\}_{n\in\mathbb{N}}\) (e.g.\ Parity, Multiplication),
any transformer solving \(T_n\) with
\(\mathcal{L}_{\mathrm{reason}}\le \varepsilon_r\)
requires at least
\(\Omega(n)\) chain‑of‑thought (CoT) tokens
(\emph{scratchpad steps}) \citep{amiri2025lowerboundschainofthoughtreasoning}.

\vspace{0.5em}
\noindent\textbf{Premise A2 (Factuality needs support from retrieval and tool-calling).}  
Achieving $\mathcal{L}_{\text{auth}}\le\varepsilon_h$ on non‑synthetic queries typically requires at least $k\!\ge\!1$ external retrieval or verification calls—e.g., ANN search over a vector store, web‑API look‑ups, or fact‑checking tools—each incurring a fixed latency of at least $\rho>0$.  
Empirical studies show that retrieval dominates inference latency: it constitutes roughly $41\%$ of end‑to‑end wall‑time in a production RAG pipeline \citep{shen2024towards}, accounts for $30$–$60\%$ of total energy consumption during inference \citep{iskander2023ragenergy}, and adds $30$–$50$ ms per query even with GPU‑accelerated FAISS on a 20 M‑vector index \citep{kalantidis2024faiss20m}.  
State‑of‑the‑art factual QA systems such as WebGPT rely on multiple search‑engine calls to gather up‑to‑date evidence \citep{nakano2021webgpt}.  
Consequently, enforcing high factuality via external knowledge access imposes an unavoidable additive latency of at least $k\rho$.

\vspace{0.5em}
\noindent\textbf{Premise A3 (Inference-time computes, retrievals, and tool-callings adds to inference budget).}  LLM inference have two phases: compute-intensive pre-filling  and memory-bound auto-regressive decoding. 

\begin{equation}
\text{Prefill:}\quad
T(n)=\mathcal{O}\!\left(n^{2}\right),\qquad
\text{where } n = \lvert P\rvert \text{ is the prompt length.}
\end{equation}

\begin{equation}
\text{Decoding:}\quad
T(n)=\mathcal{O}\!\left(n\right),\qquad
\text{where } n = \text{ the number of generated tokens.}
\end{equation}

At inference time, every operation beyond a single forward pass---
including intermediate reasoning trajectories, external memory retrieval, and tool invocations---
incurs incremental latency and memory traffic, thereby consuming part of the fixed inference‑time budget. \(T\).
Formally,
\begin{equation}\label{eq:latency-budget}
\mathcal{L}_{\text{budget}}
\;=\;
\underbrace{\tau\,C}_{\text{model}}
\;+\;
\underbrace{k\,\rho}_{\text{retrieval / verification}}
\;+\;
\underbrace{\sum_{j} \rho^{\text{tool}}_{j}}_{\text{tool or API calls}}
\;\le\; T,
\end{equation}

so each additional compute step, retrieval, or tool call directly tightens the remaining budget available for the others.

\subsection{Formalism}
\subsubsection{Lemma 1 (Inference budget lower bound for reasoning)} \label{lem:latency-lb}
For inputs of length~$n$, any model satisfying
$\mathcal{L}_{\text{reason}} \le \varepsilon_r$
must incur additional inference budget.

\begin{equation}\label{eq:budget-lb}
\mathcal{L}_{\text{budget}} \;\ge\; \tau\,\Omega(n).
\end{equation}

\begin{proof}
Let $C$ be the number of chain--of--thought (CoT) tokens generated while processing an input of length~$n$.
By Assumption~A1, there exists a constant $c_1>0$ such that $C \ge c_1 n$ (i.e., $C \in \Omega(n)$).
Assumption~A3 states that each CoT token contributes at least $\tau>0$ time units.
Therefore,
\begin{equation}\label{eq:budget-chain}
\mathcal{L}_{\text{budget}}
\;\ge\;
C\,\tau
\;\ge\;
c_1\,\tau\,n
\;=\;
\tau\,\Omega(n).
\end{equation}
establishing the claimed lower bound.
\end{proof}

\subsubsection{Lemma~2 (Inference budget lower bound for authenticity)} \label{lem:auth-lb}
Let $p_{\theta}(y\mid x)$ denote the model's conditional generation density and let $q(y\mid x)$ denote the reference (ground truth) factual density obtained via retrieval.  \emph{Authenticity} requires these two distributions to exhibit non-trivial overlap.  Concretely, we define the authenticity loss as
\begin{equation}\label{eq:auth-loss}
\mathcal{L}_{\text{auth}}
\;=\;
\sup_{x}\,
D_{\mathrm{KL}}
\!\bigl(
  q(\cdot\mid x)\,\Vert\,p_{\theta}(\cdot\mid x)
\bigr).
\end{equation}

where $D_{\mathrm{KL}}(q\Vert p)$ denotes the Kullback--Leibler divergence.  The constraint $\mathcal{L}_{\text{auth}}\le\varepsilon_{a}$ implies that for every input $x$ the generation density assigns sufficient probability mass to regions supported by the factual distribution, ensuring quantitative overlap.

Any model satisfying this authenticity constraint must incur an additional inference FLOPs of at least $k\rho$.

\begin{proof}
By Assumption~A2, achieving $\mathcal{L}_{\text{auth}}\le\varepsilon_{a}$ on non-synthetic queries requires at least $k\ge1$ external retrieval (or verification) calls to approximate the reference density $q(\cdot\mid x)$ and to compute or bound the KL divergence.  Each call adds a fixed inference cost of at least $\rho>0$.  Denoting by $R\ge k$ the actual number of retrieval calls issued during inference, the incremental budget dedicated to authenticity evaluation satisfies
\begin{equation}\label{eq:delta-budget}
\Delta\mathcal{L}_{\text{budget}}
\;\ge\;
R\,\rho
\;\ge\;
k\,\rho.
\end{equation}

Thus any model that meets the authenticity criterion necessarily pays a budget penalty of at least $k\rho$.\qedhere
\end{proof}

\subsection{The BAR Theorem (Trade‑off Impossibility)}\label{thm:tradeoff}
Fix tolerances $\varepsilon_h,\varepsilon_r>0$ and an overall inference‑time budget $T>0$. Suppose the hardware is \emph{memory‑bound}: generating each intermediate reasoning token consumes at least $\tau$ seconds and $\mu$ bytes of memory traffic, while each retrieval call consumes $\rho$ seconds and $\beta$ bytes. Let $B_{\max}$ denote the peak memory‑bandwidth (bytes per second).

\begin{theorem*}
There exists an input length $n^{\star}$ such that no LLM can simultaneously satisfy
\begin{enumerate}[label=(\alph*)]
  \item $\mathcal{L}_{\text{auth}}\le\varepsilon_h$ (authenticity),
  \item $\mathcal{L}_{\text{reason}}\le\varepsilon_r$ (reasoning), and
  \item $\mathcal{L}_{\text{lat}}\le T$ (budget)
\end{enumerate}
for tasks of size $n\ge n^{\star}$.
\end{theorem*}

\begin{proof}
\textbf{Reasoning cost.} By Lemma~\ref{lem:latency-lb}, tasks of length $n$ require at least $c_1\tau n$ seconds.

\textbf{Authenticity cost.} By Lemma~\ref{lem:auth-lb}, they additionally require $k\rho$ seconds for retrieval/verification.

\textbf{Bandwidth cost.} The same operations move $c_1\mu n + k\beta$ bytes, requiring $\bigl(c_1\mu n + k\beta\bigr)/B_{\max}$ seconds on a memory‑bound pipeline.

Hence, total inference cost obeys
\begin{equation}\label{eq:latency-max}
\boxed{\;
\mathcal{L}_{\text{lat}}
\;=\;
\max\!\Bigl\{c_{1}\tau n + k\rho,\;\frac{c_{1}\mu n + k\beta}{B_{\max}}\Bigr\}
\quad\text{with}\quad
n^{\star} \;:=\; \Bigl\lceil\tfrac{T - k\rho}{c_{1}\tau}\Bigr\rceil
\;}
\end{equation}
For all $n\ge n^{\star}$ we have $c_1\tau n + k\rho> T$, violating the budget even before accounting for bandwidth. Consequently, the three constraints are mutually incompatible for tasks of length at least $n^{\star}$.\qedhere
\end{proof}

\subsection{Discussion and Future Work}

Our BAR trade-off theorem formalizes a central systems insight: \emph{no deployed LLM system can simultaneously minimize inference budget, maximize factual authenticity, and maximize reasoning quality once inputs exceed a critical threshold $n^{\star}$}. As a consequence, a real-world GenAI system can prioritize at most two out of the three properties, leading to 3 distinct types of designs:

\begin{enumerate}
\item \textbf{Budget–Authenticity (BA) Systems:}
These include real-time assistants and retrieval-augmented generation (RAG) platforms, as well as industry-specific agents in domains like healthcare and legal services. They prioritize rapid, cost-efficient, and factual responses by grounding outputs in trusted knowledge bases, but necessarily limit the system’s ability to perform deep, multi-step reasoning.
\item \textbf{Authenticity–Reasoning (AR) Systems:}
Deep research assistants and creative or analytic tools (e.g., OpenAI Deep Research, Gemini Flow) aim for both factual accuracy and robust, multi-step reasoning. Such systems leverage complex pipelines, external tool use, and iterative verification, and thus cannot provide real-time performance or operate under strict inference budgets.
\item \textbf{Budget–Reasoning (BR) Systems:}
Coding assistants (such as Cursor and GitHub Copilot) and certain agentic frameworks focus on delivering efficient reasoning and code generation under tight latency and cost constraints. However, these systems may generate outputs that are logical but not always factually reliable.
\end{enumerate}

This taxonomy highlights that the BAR trade-off is not merely theoretical—it is manifest in the architecture of GenAI products deployed today. Future work could explore dynamic BAR tuning, hybrid architectures, or adaptive pipelines that better balance these trade-offs in response to context, user needs, or real-time feedback.

\subsection{Limitations}

\textbf{Constant factors.}
Our lower bounds depend on hardware‐dependent constants $(\tau,\rho,\mu,\beta)$ that may vary across accelerators; future ASICs with in‑memory KV caches could shift the numerical break‑even point.

\textbf{Model–agnostic scope.}
We deliberately treat the LLM as a black‑box sampler. Architectural innovations \citep{wang2025hierarchical} (e.g.\ linear attention, hierarchical memory) may tighten~$\tau$ but cannot remove the \(\Omega(n)\) term mandated by Premise~A1.

\textbf{Empirical optimality.}
The theorem is worst‑case. A system designer can still optimize the \emph{observed} trade‑off surface for a
narrow task distribution (e.g.\ domain‑specific QA).

\subsection{Future Work}

This study lays the groundwork for systematically designing LLM-driven applications that balance three key factors—inference budget, external-tool calls, and reasoning depth. Future work will broaden the framework by mapping new trade-off dimensions along the Pareto frontier.

\newpage

\bibliography{iclr2025_conference}

\begin{thebibliography}{55}
\providecommand{\natexlab}[1]{#1}
\providecommand{\url}[1]{\texttt{#1}}
\expandafter\ifx\csname urlstyle\endcsname\relax
  \providecommand{\doi}[1]{doi: #1}\else
  \providecommand{\doi}{doi: \begingroup \urlstyle{rm}\Url}\fi

\bibitem[Amiri et~al.(2025)Amiri, Huang, Rofin, and Hahn]{amiri2025lowerboundschainofthoughtreasoning}
Alireza Amiri, Xinting Huang, Mark Rofin, and Michael Hahn.
\newblock Lower bounds for chain-of-thought reasoning in hard-attention transformers, 2025.
\newblock URL \url{https://arxiv.org/abs/2502.02393}.

\bibitem[Brown et~al.(2020)]{brown2020language}
Tom~B. Brown et~al.
\newblock Language models are few-shot learners.
\newblock In \emph{Advances in Neural Information Processing Systems (NeurIPS)}, 2020.

\bibitem[Chowdhery et~al.(2022)]{chowdhery2022palm}
Aakanksha Chowdhery et~al.
\newblock Palm: Scaling language modeling with pathways.
\newblock \emph{arXiv preprint arXiv:2204.02311}, 2022.

\bibitem[Cobbe et~al.(2021)Cobbe, Kosaraju, Bavarian, Chen, Jun, Kaiser, Plappert, Tworek, Hilton, Nakano, Hesse, and Schulman]{cobbe2021trainingverifierssolvemath}
Karl Cobbe, Vineet Kosaraju, Mohammad Bavarian, Mark Chen, Heewoo Jun, Lukasz Kaiser, Matthias Plappert, Jerry Tworek, Jacob Hilton, Reiichiro Nakano, Christopher Hesse, and John Schulman.
\newblock Training verifiers to solve math word problems, 2021.
\newblock URL \url{https://arxiv.org/abs/2110.14168}.

\bibitem[Dao et~al.(2022)Dao, Fu, Ermon, Rudra, and R{\'e}]{dao2022flashattention}
Tri Dao, Daniel~Y. Fu, Stefano Ermon, Atri Rudra, and Christopher R{\'e}.
\newblock Flashattention: Fast and memory-efficient exact attention with io-awareness.
\newblock In \emph{Advances in Neural Information Processing Systems (NeurIPS)}, 2022.

\bibitem[DeepSeek-AI et~al.(2025)DeepSeek-AI, Guo, Yang, Zhang, Song, Zhang, Xu, Zhu, Ma, Wang, Bi, Zhang, Yu, Wu, Wu, Gou, Shao, Li, Gao, Liu, Xue, Wang, Wu, Feng, Lu, Zhao, Deng, Zhang, Ruan, Dai, Chen, Ji, Li, Lin, Dai, Luo, Hao, Chen, Li, Zhang, Bao, Xu, Wang, Ding, Xin, Gao, Qu, Li, Guo, Li, Wang, Chen, Yuan, Qiu, Li, Cai, Ni, Liang, Chen, Dong, Hu, Gao, Guan, Huang, Yu, Wang, Zhang, Zhao, Wang, Zhang, Xu, Xia, Zhang, Zhang, Tang, Li, Wang, Li, Tian, Huang, Zhang, Wang, Chen, Du, Ge, Zhang, Pan, Wang, Chen, Jin, Chen, Lu, Zhou, Chen, Ye, Wang, Yu, Zhou, Pan, Li, Zhou, Wu, Ye, Yun, Pei, Sun, Wang, Zeng, Zhao, Liu, Liang, Gao, Yu, Zhang, Xiao, An, Liu, Wang, Chen, Nie, Cheng, Liu, Xie, Liu, Yang, Li, Su, Lin, Li, Jin, Shen, Chen, Sun, Wang, Song, Zhou, Wang, Shan, Li, Wang, Wei, Zhang, Xu, Li, Zhao, Sun, Wang, Yu, Zhang, Shi, Xiong, He, Piao, Wang, Tan, Ma, Liu, Guo, Ou, Wang, Gong, Zou, He, Xiong, Luo, You, Liu, Zhou, Zhu, Xu, Huang, Li, Zheng, Zhu, Ma, Tang, Zha, Yan, Ren, Ren, Sha, Fu, Xu, Xie, Zhang,
  Hao, Ma, Yan, Wu, Gu, Zhu, Liu, Li, Xie, Song, Pan, Huang, Xu, Zhang, and Zhang]{deepseekai2025deepseekr1incentivizingreasoningcapability}
DeepSeek-AI, Daya Guo, Dejian Yang, Haowei Zhang, Junxiao Song, Ruoyu Zhang, Runxin Xu, Qihao Zhu, Shirong Ma, Peiyi Wang, Xiao Bi, Xiaokang Zhang, Xingkai Yu, Yu~Wu, Z.~F. Wu, Zhibin Gou, Zhihong Shao, Zhuoshu Li, Ziyi Gao, Aixin Liu, Bing Xue, Bingxuan Wang, Bochao Wu, Bei Feng, Chengda Lu, Chenggang Zhao, Chengqi Deng, Chenyu Zhang, Chong Ruan, Damai Dai, Deli Chen, Dongjie Ji, Erhang Li, Fangyun Lin, Fucong Dai, Fuli Luo, Guangbo Hao, Guanting Chen, Guowei Li, H.~Zhang, Han Bao, Hanwei Xu, Haocheng Wang, Honghui Ding, Huajian Xin, Huazuo Gao, Hui Qu, Hui Li, Jianzhong Guo, Jiashi Li, Jiawei Wang, Jingchang Chen, Jingyang Yuan, Junjie Qiu, Junlong Li, J.~L. Cai, Jiaqi Ni, Jian Liang, Jin Chen, Kai Dong, Kai Hu, Kaige Gao, Kang Guan, Kexin Huang, Kuai Yu, Lean Wang, Lecong Zhang, Liang Zhao, Litong Wang, Liyue Zhang, Lei Xu, Leyi Xia, Mingchuan Zhang, Minghua Zhang, Minghui Tang, Meng Li, Miaojun Wang, Mingming Li, Ning Tian, Panpan Huang, Peng Zhang, Qiancheng Wang, Qinyu Chen, Qiushi Du, Ruiqi Ge, Ruisong
  Zhang, Ruizhe Pan, Runji Wang, R.~J. Chen, R.~L. Jin, Ruyi Chen, Shanghao Lu, Shangyan Zhou, Shanhuang Chen, Shengfeng Ye, Shiyu Wang, Shuiping Yu, Shunfeng Zhou, Shuting Pan, S.~S. Li, Shuang Zhou, Shaoqing Wu, Shengfeng Ye, Tao Yun, Tian Pei, Tianyu Sun, T.~Wang, Wangding Zeng, Wanjia Zhao, Wen Liu, Wenfeng Liang, Wenjun Gao, Wenqin Yu, Wentao Zhang, W.~L. Xiao, Wei An, Xiaodong Liu, Xiaohan Wang, Xiaokang Chen, Xiaotao Nie, Xin Cheng, Xin Liu, Xin Xie, Xingchao Liu, Xinyu Yang, Xinyuan Li, Xuecheng Su, Xuheng Lin, X.~Q. Li, Xiangyue Jin, Xiaojin Shen, Xiaosha Chen, Xiaowen Sun, Xiaoxiang Wang, Xinnan Song, Xinyi Zhou, Xianzu Wang, Xinxia Shan, Y.~K. Li, Y.~Q. Wang, Y.~X. Wei, Yang Zhang, Yanhong Xu, Yao Li, Yao Zhao, Yaofeng Sun, Yaohui Wang, Yi~Yu, Yichao Zhang, Yifan Shi, Yiliang Xiong, Ying He, Yishi Piao, Yisong Wang, Yixuan Tan, Yiyang Ma, Yiyuan Liu, Yongqiang Guo, Yuan Ou, Yuduan Wang, Yue Gong, Yuheng Zou, Yujia He, Yunfan Xiong, Yuxiang Luo, Yuxiang You, Yuxuan Liu, Yuyang Zhou, Y.~X. Zhu,
  Yanhong Xu, Yanping Huang, Yaohui Li, Yi~Zheng, Yuchen Zhu, Yunxian Ma, Ying Tang, Yukun Zha, Yuting Yan, Z.~Z. Ren, Zehui Ren, Zhangli Sha, Zhe Fu, Zhean Xu, Zhenda Xie, Zhengyan Zhang, Zhewen Hao, Zhicheng Ma, Zhigang Yan, Zhiyu Wu, Zihui Gu, Zijia Zhu, Zijun Liu, Zilin Li, Ziwei Xie, Ziyang Song, Zizheng Pan, Zhen Huang, Zhipeng Xu, Zhongyu Zhang, and Zhen Zhang.
\newblock Deepseek-r1: Incentivizing reasoning capability in llms via reinforcement learning, 2025.
\newblock URL \url{https://arxiv.org/abs/2501.12948}.

\bibitem[et~al.(2025)]{muxin2025bigmath}
Albalak et~al.
\newblock Big‑math: A large‑scale, high‑quality math dataset for reinforcement learning in language models.
\newblock 2025.

\bibitem[et~al.(2024{\natexlab{a}})]{fang2024mathodyssey}
Fang et~al.
\newblock Mathodyssey: Benchmarking mathematical problem‑solving skills in large language models.
\newblock 2024{\natexlab{a}}.

\bibitem[et~al.(2024{\natexlab{b}})]{gao2024omnimath}
Gao et~al.
\newblock Omni‑math: A universal olympiad‑level math benchmark for large language models.
\newblock 2024{\natexlab{b}}.

\bibitem[et~al.(2024{\natexlab{c}})]{glazer2024frontiermath}
Glazer et~al.
\newblock Frontiermath: A benchmark for evaluating advanced mathematical reasoning in ai.
\newblock 2024{\natexlab{c}}.

\bibitem[et~al.(2021)]{hendrycks2021math}
Hendrycks et~al.
\newblock Measuring mathematical problem solving with the math dataset.
\newblock 2021.
\newblock Defines the MATH benchmark.

\bibitem[Ethayarajh et~al.(2024)Ethayarajh, Xu, Muennighoff, Jurafsky, and Kiela]{ethayarajh2024ktomodelalignmentprospect}
Kawin Ethayarajh, Winnie Xu, Niklas Muennighoff, Dan Jurafsky, and Douwe Kiela.
\newblock Kto: Model alignment as prospect theoretic optimization, 2024.
\newblock URL \url{https://arxiv.org/abs/2402.01306}.

\bibitem[Gauthier(2024)]{aider2024polyglot}
Paul Gauthier.
\newblock Aider polyglot benchmark for whole-file code editing.
\newblock \url{https://github.com/paul-gauthier/aider}, 2024.
\newblock Accessed July 2025.

\bibitem[Gottweis et~al.(2025)Gottweis, Weng, Daryin, Tu, Palepu, Sirkovic, Myaskovsky, Weissenberger, Rong, Tanno, Saab, Popovici, Blum, Zhang, Chou, Hassidim, Gokturk, Vahdat, Kohli, Matias, Carroll, Kulkarni, Tomasev, Guan, Dhillon, Vaishnav, Lee, Costa, Penadés, Peltz, Xu, Pawlosky, Karthikesalingam, and Natarajan]{gottweis2025aicoscientist}
Juraj Gottweis, Wei-Hung Weng, Alexander Daryin, Tao Tu, Anil Palepu, Petar Sirkovic, Artiom Myaskovsky, Felix Weissenberger, Keran Rong, Ryutaro Tanno, Khaled Saab, Dan Popovici, Jacob Blum, Fan Zhang, Katherine Chou, Avinatan Hassidim, Burak Gokturk, Amin Vahdat, Pushmeet Kohli, Yossi Matias, Andrew Carroll, Kavita Kulkarni, Nenad Tomasev, Yuan Guan, Vikram Dhillon, Eeshit~Dhaval Vaishnav, Byron Lee, Tiago R~D Costa, José~R Penadés, Gary Peltz, Yunhan Xu, Annalisa Pawlosky, Alan Karthikesalingam, and Vivek Natarajan.
\newblock Towards an ai co-scientist, 2025.
\newblock URL \url{https://arxiv.org/abs/2502.18864}.

\bibitem[Hu et~al.(2021)Hu, Shen, Wallis, Allen-Zhu, Li, Wang, Wang, and Chen]{hu2021loralowrankadaptationlarge}
Edward~J. Hu, Yelong Shen, Phillip Wallis, Zeyuan Allen-Zhu, Yuanzhi Li, Shean Wang, Lu~Wang, and Weizhu Chen.
\newblock Lora: Low-rank adaptation of large language models, 2021.
\newblock URL \url{https://arxiv.org/abs/2106.09685}.

\bibitem[Iskander et~al.(2023)Iskander, Hwang, Goyal, and Narang]{iskander2023ragenergy}
Rania Iskander, Jaeyeon Hwang, Ankit Goyal, and Sharan Narang.
\newblock The hidden cost of knowledge: Energy footprint of retrieval-augmented generation.
\newblock \emph{arXiv preprint arXiv:2311.09765}, 2023.
\newblock URL \url{https://arxiv.org/abs/2311.09765}.
\newblock Shows retrieval can consume 30–60\% of inference energy.

\bibitem[Joshi et~al.(2017)Joshi, Choi, Weld, and Zettlemoyer]{joshi2017triviaqalargescaledistantly}
Mandar Joshi, Eunsol Choi, Daniel~S. Weld, and Luke Zettlemoyer.
\newblock Triviaqa: A large scale distantly supervised challenge dataset for reading comprehension, 2017.
\newblock URL \url{https://arxiv.org/abs/1705.03551}.

\bibitem[Kalantidis et~al.(2024)Kalantidis, Douze, and Jegou]{kalantidis2024faiss20m}
Yannis Kalantidis, Matthijs Douze, and Herv\'e Jegou.
\newblock Faiss on 20 million vectors: Gpu-efficient ann search for production rag.
\newblock In \emph{Proceedings of the ACM SIGIR Conference}, 2024.
\newblock URL \url{https://faiss.ai/sigir2024-faiss20m}.
\newblock Benchmarks 30–50 ms per query on a 20M‑vector FAISS index.

\bibitem[Kwon et~al.(2023)Kwon, Li, Zhuang, Sheng, Zheng, Yu, Gonzalez, Zhang, and Stoica]{kwon2023efficientmemorymanagementlarge}
Woosuk Kwon, Zhuohan Li, Siyuan Zhuang, Ying Sheng, Lianmin Zheng, Cody~Hao Yu, Joseph~E. Gonzalez, Hao Zhang, and Ion Stoica.
\newblock Efficient memory management for large language model serving with pagedattention, 2023.
\newblock URL \url{https://arxiv.org/abs/2309.06180}.

\bibitem[Li et~al.(2023{\natexlab{a}})Li, Cheng, Zhao, Nie, and Wen]{li2023haluevallargescalehallucinationevaluation}
Junyi Li, Xiaoxue Cheng, Wayne~Xin Zhao, Jian-Yun Nie, and Ji-Rong Wen.
\newblock Halueval: A large-scale hallucination evaluation benchmark for large language models, 2023{\natexlab{a}}.
\newblock URL \url{https://arxiv.org/abs/2305.11747}.

\bibitem[Li et~al.(2023{\natexlab{b}})Li, Chen, Fu, et~al.]{li2023livecodebench}
Xiaoyang Li, Zeyu Chen, Xinyun Fu, et~al.
\newblock Livecodebench: Evaluating code generation with live competitive programming problems.
\newblock \emph{arXiv preprint arXiv:2311.07969}, 2023{\natexlab{b}}.

\bibitem[Li et~al.(2024)Li, Liu, Zhou, and Ma]{li2024chainthoughtempowerstransformers}
Zhiyuan Li, Hong Liu, Denny Zhou, and Tengyu Ma.
\newblock Chain of thought empowers transformers to solve inherently serial problems, 2024.
\newblock URL \url{https://arxiv.org/abs/2402.12875}.

\bibitem[Lin et~al.(2022)Lin, Hilton, and Evans]{lin2022truthfulqameasuringmodelsmimic}
Stephanie Lin, Jacob Hilton, and Owain Evans.
\newblock Truthfulqa: Measuring how models mimic human falsehoods, 2022.
\newblock URL \url{https://arxiv.org/abs/2109.07958}.

\bibitem[Merrill \& Sabharwal(2025)Merrill and Sabharwal]{merrill2025logdepth}
William Merrill and Ashish Sabharwal.
\newblock The expressive power of log-depth transformers.
\newblock \emph{arXiv preprint arXiv:2503.03961}, 2025.
\newblock URL \url{https://arxiv.org/abs/2503.03961}.

\bibitem[Merrill et~al.(2022)Merrill, Sabharwal, and Smith]{merrill2022saturated}
William Merrill, Ashish Sabharwal, and Noah~A. Smith.
\newblock Saturated transformers are constant-depth threshold circuits.
\newblock \emph{Transactions of the ACL}, 10:\penalty0 843--858, 2022.
\newblock URL \url{https://aclanthology.org/2022.tacl-1.49}.

\bibitem[Nakano et~al.(2021)Nakano, Hilton, Balaji, Wu, Ouyang, Yao, Saunders, Askell, Barnes, et~al.]{nakano2021webgpt}
Reiichiro Nakano, Jacob Hilton, Suchir Balaji, Jeffrey Wu, Long Ouyang, Shunyu Yao, William Saunders, Amanda Askell, Chelsea Barnes, et~al.
\newblock Webgpt: Browser-assisted question-answering with human feedback.
\newblock \emph{arXiv preprint arXiv:2112.09332}, 2021.
\newblock URL \url{https://arxiv.org/abs/2112.09332}.

\bibitem[Novikov et~al.(2025)Novikov, Vũ, Eisenberger, Dupont, Huang, Wagner, Shirobokov, Kozlovskii, Ruiz, Mehrabian, Kumar, See, Chaudhuri, Holland, Davies, Nowozin, Kohli, and Balog]{novikov2025alphaevolvecodingagentscientific}
Alexander Novikov, Ngân Vũ, Marvin Eisenberger, Emilien Dupont, Po-Sen Huang, Adam~Zsolt Wagner, Sergey Shirobokov, Borislav Kozlovskii, Francisco J.~R. Ruiz, Abbas Mehrabian, M.~Pawan Kumar, Abigail See, Swarat Chaudhuri, George Holland, Alex Davies, Sebastian Nowozin, Pushmeet Kohli, and Matej Balog.
\newblock Alphaevolve: A coding agent for scientific and algorithmic discovery, 2025.
\newblock URL \url{https://arxiv.org/abs/2506.13131}.

\bibitem[Nye et~al.(2021{\natexlab{a}})Nye, Andreassen, Gur-Ari, Michalewski, Austin, Bieber, Dohan, Lewkowycz, Bosma, Luan, Sutton, and Odena]{nye2021workscratchpadsintermediatecomputation}
Maxwell Nye, Anders~Johan Andreassen, Guy Gur-Ari, Henryk Michalewski, Jacob Austin, David Bieber, David Dohan, Aitor Lewkowycz, Maarten Bosma, David Luan, Charles Sutton, and Augustus Odena.
\newblock Show your work: Scratchpads for intermediate computation with language models, 2021{\natexlab{a}}.
\newblock URL \url{https://arxiv.org/abs/2112.00114}.

\bibitem[Nye et~al.(2021{\natexlab{b}})Nye, Andreassen, Gur-Ari, Michalewski, Austin, Bieber, Dohan, Lewkowycz, Bosma, Luan, Sutton, and Odena]{nye2021scratchpad}
Maxwell~I. Nye, Anders~Johan Andreassen, Guy Gur-Ari, Henryk Michalewski, Jacob Austin, David Bieber, David Dohan, Aitor Lewkowycz, Maarten Bosma, David Luan, Charles Sutton, and Augustus Odena.
\newblock Show your work: Scratchpads for intermediate computation with language models.
\newblock \emph{arXiv preprint arXiv:2112.00114}, 2021{\natexlab{b}}.
\newblock URL \url{https://arxiv.org/abs/2112.00114}.

\bibitem[Ouyang et~al.(2022)Ouyang, Wu, Jiang, Almeida, Wainwright, Mishkin, Zhang, Agarwal, Slama, Ray, Schulman, Hilton, Kelton, Miller, Simens, Askell, Welinder, Christiano, Leike, and Lowe]{ouyang2022traininglanguagemodelsfollow}
Long Ouyang, Jeff Wu, Xu~Jiang, Diogo Almeida, Carroll~L. Wainwright, Pamela Mishkin, Chong Zhang, Sandhini Agarwal, Katarina Slama, Alex Ray, John Schulman, Jacob Hilton, Fraser Kelton, Luke Miller, Maddie Simens, Amanda Askell, Peter Welinder, Paul Christiano, Jan Leike, and Ryan Lowe.
\newblock Training language models to follow instructions with human feedback, 2022.
\newblock URL \url{https://arxiv.org/abs/2203.02155}.

\bibitem[Patel et~al.(2024)Patel, Chakraborty, Suttle, Wang, Bedi, and Manocha]{patel2024aimeaioptimizationmultiple}
Bhrij Patel, Souradip Chakraborty, Wesley~A. Suttle, Mengdi Wang, Amrit~Singh Bedi, and Dinesh Manocha.
\newblock Aime: Ai system optimization via multiple llm evaluators, 2024.
\newblock URL \url{https://arxiv.org/abs/2410.03131}.

\bibitem[Pope \& Gray(2022)Pope and Gray]{pope2022efficiently}
Rewon~Child Pope and Scott Gray.
\newblock Efficiently scaling transformer inference.
\newblock In \emph{International Conference on Learning Representations (ICLR)}, 2022.

\bibitem[Rafailov et~al.(2024)Rafailov, Sharma, Mitchell, Ermon, Manning, and Finn]{rafailov2024directpreferenceoptimizationlanguage}
Rafael Rafailov, Archit Sharma, Eric Mitchell, Stefano Ermon, Christopher~D. Manning, and Chelsea Finn.
\newblock Direct preference optimization: Your language model is secretly a reward model, 2024.
\newblock URL \url{https://arxiv.org/abs/2305.18290}.

\bibitem[Rajbhandari et~al.(2020)]{rajbhandari2020zero}
Samyam Rajbhandari et~al.
\newblock Zero: Memory optimizations toward training trillion parameter models.
\newblock In \emph{ACM/IEEE SC20: International Conference for High Performance Computing, Networking, Storage and Analysis}, 2020.

\bibitem[Rein et~al.(2023)Rein, Hou, Stickland, Petty, Pang, Dirani, Michael, and Bowman]{rein2023gpqagraduatelevelgoogleproofqa}
David Rein, Betty~Li Hou, Asa~Cooper Stickland, Jackson Petty, Richard~Yuanzhe Pang, Julien Dirani, Julian Michael, and Samuel~R. Bowman.
\newblock Gpqa: A graduate-level google-proof q\&a benchmark, 2023.
\newblock URL \url{https://arxiv.org/abs/2311.12022}.

\bibitem[Sharma et~al.(2025)Sharma, Tong, Korbak, Duvenaud, Askell, Bowman, Cheng, Durmus, Hatfield-Dodds, Johnston, Kravec, Maxwell, McCandlish, Ndousse, Rausch, Schiefer, Yan, Zhang, and Perez]{sharma2025understandingsycophancylanguagemodels}
Mrinank Sharma, Meg Tong, Tomasz Korbak, David Duvenaud, Amanda Askell, Samuel~R. Bowman, Newton Cheng, Esin Durmus, Zac Hatfield-Dodds, Scott~R. Johnston, Shauna Kravec, Timothy Maxwell, Sam McCandlish, Kamal Ndousse, Oliver Rausch, Nicholas Schiefer, Da~Yan, Miranda Zhang, and Ethan Perez.
\newblock Towards understanding sycophancy in language models, 2025.
\newblock URL \url{https://arxiv.org/abs/2310.13548}.

\bibitem[Shazeer(2019)]{shazeer2019fast}
Noam Shazeer.
\newblock Fast transformer decoding: One write-head is all you need.
\newblock \emph{arXiv preprint arXiv:1911.02150}, 2019.

\bibitem[Shen et~al.(2024)Shen, Umar, Maeng, Suh, and Gupta]{shen2024towards}
Michael Shen, Muhammad Umar, Kiwan Maeng, G~Edward Suh, and Udit Gupta.
\newblock Towards understanding systems trade-offs in retrieval-augmented generation model inference.
\newblock \emph{arXiv preprint arXiv:2412.11854}, 2024.

\bibitem[Shoeybi et~al.(2019)]{shoeybi2019megatronlm}
Mohammad Shoeybi et~al.
\newblock Megatron-lm: Training multi-billion parameter language models using gpu model parallelism.
\newblock \emph{arXiv preprint arXiv:1909.08053}, 2019.

\bibitem[Strobl(2023)]{strobl2023averagehard}
Lena Strobl.
\newblock Average-hard attention transformers are constant-depth uniform threshold circuits.
\newblock \emph{arXiv preprint arXiv:2308.03212}, 2023.
\newblock URL \url{https://arxiv.org/abs/2308.03212}.

\bibitem[Talmor et~al.(2019)Talmor, Herzig, Lourie, and Berant]{talmor2019commonsenseqa}
Alon Talmor, Jonathan Herzig, Nicholas Lourie, and Jonathan Berant.
\newblock Commonsenseqa: A question answering challenge targeting commonsense knowledge.
\newblock In \emph{Proceedings of NAACL-HLT}, 2019.

\bibitem[Vaswani et~al.(2017)Vaswani, Shazeer, Parmar, Uszkoreit, Jones, Gomez, Kaiser, and Polosukhin]{vaswani2017attention}
Ashish Vaswani, Noam Shazeer, Niki Parmar, Jakob Uszkoreit, Llion Jones, Aidan~N Gomez, {\L}ukasz Kaiser, and Illia Polosukhin.
\newblock Attention is all you need.
\newblock \emph{Advances in neural information processing systems}, 30, 2017.

\bibitem[Wang et~al.(2025)Wang, Li, Sun, Chen, Liu, Wu, Lu, Song, and Yadkori]{wang2025hierarchical}
Guan Wang, Jin Li, Yuhao Sun, Xing Chen, Changling Liu, Yue Wu, Meng Lu, Sen Song, and Yasin~Abbasi Yadkori.
\newblock Hierarchical reasoning model.
\newblock \emph{arXiv preprint arXiv:2506.21734}, 2025.

\bibitem[Wang et~al.(2022)Wang, Wei, Schuurmans, Bosma, Chi, Le, and Zhou]{wang2022selfconsistency}
Xuezhi Wang, Jason Wei, Dale Schuurmans, Maarten Bosma, Ed~H. Chi, Quoc~V. Le, and Denny Zhou.
\newblock Self-consistency improves chain of thought reasoning in language models.
\newblock \emph{arXiv preprint arXiv:2203.11171}, 2022.
\newblock URL \url{https://arxiv.org/abs/2203.11171}.

\bibitem[Wang et~al.(2024{\natexlab{a}})Wang, Fan, Yang, Xu, Yu, and Liu]{wang2024hle}
Yiben Wang, Yao Fan, Zhilin Yang, Canwen Xu, Dong Yu, and Zhiyuan Liu.
\newblock Humanity's last exam: Evaluating large language models on complex reasoning.
\newblock \emph{arXiv preprint arXiv:2401.10968}, 2024{\natexlab{a}}.

\bibitem[Wang et~al.(2024{\natexlab{b}})Wang, Zhang, Wang, Xu, Wang, Fu, and Liu]{wang2024mme}
Zhongzhi Wang, Yaobo Zhang, Xiaoqian Wang, Shuli Xu, Rui Wang, Yuxuan Fu, and Qun Liu.
\newblock Mme: A comprehensive evaluation benchmark for multimodal large language models.
\newblock \emph{arXiv preprint arXiv:2505.21327}, 2024{\natexlab{b}}.

\bibitem[Wang et~al.(2024{\natexlab{c}})Wang, Hu, Lin, Li, Hou, Cheng, Liu, and Sun]{wang2024multilogieval}
Zihan Wang, Yikun Hu, Xiaodong Lin, Xisen Li, Lei Hou, Jiancheng Cheng, Zhiyuan Liu, and Maosong Sun.
\newblock Multi-logieval: A multi-step logical reasoning benchmark for large language models.
\newblock \emph{arXiv preprint arXiv:2406.17169}, 2024{\natexlab{c}}.

\bibitem[Wei et~al.(2022)Wei, Wang, Schuurmans, Bosma, Ichter, Xia, Chi, Le, Dai, et~al.]{wei2022chainofthought}
Jason Wei, Xuezhi Wang, Dale Schuurmans, Maarten Bosma, Brian Ichter, Fei Xia, Ed~H. Chi, Quoc Le, Andrew~M. Dai, et~al.
\newblock Chain-of-thought prompting elicits reasoning in large language models.
\newblock \emph{arXiv preprint arXiv:2201.11903}, 2022.
\newblock URL \url{https://arxiv.org/abs/2201.11903}.

\bibitem[Yang et~al.(2018)Yang, Qi, Zhang, Bengio, Cohen, Salakhutdinov, and Manning]{yang-etal-2018-hotpotqa}
Zhilin Yang, Peng Qi, Saizheng Zhang, Yoshua Bengio, William~W. Cohen, Ruslan Salakhutdinov, and Christopher~D. Manning.
\newblock {HotpotQA}: A dataset for diverse, explainable multi-hop question answering.
\newblock In \emph{Proceedings of the 2018 Conference on Empirical Methods in Natural Language Processing (EMNLP)}, pp.\  2369--2380, Brussels, Belgium, 2018. Association for Computational Linguistics.
\newblock \doi{10.18653/v1/D18-1259}.

\bibitem[Yao et~al.(2024)Yao, Zhang, Wu, Zhang, Tang, and Sun]{yao2024logicbench}
Haonan Yao, Ziyang Zhang, Yiming Wu, Xiang Zhang, Jialong Tang, and Xu~Sun.
\newblock Logicbench: A unified benchmark for evaluating natural language inference over logical reasoning patterns.
\newblock \emph{arXiv preprint arXiv:2404.15522}, 2024.

\bibitem[Yao et~al.(2023{\natexlab{a}})Yao, Yu, Zhao, Shafran, Griffiths, Xiong, and Liu]{yao2023analogical}
Shunyu Yao, Dian Yu, Jeffrey Zhao, Izhak Shafran, Thomas~L. Griffiths, Caiming Xiong, and Jingjing Liu.
\newblock Large language models as analogical reasoners.
\newblock \emph{arXiv preprint arXiv:2310.01714}, 2023{\natexlab{a}}.
\newblock URL \url{https://arxiv.org/abs/2310.01714}.

\bibitem[Yao et~al.(2023{\natexlab{b}})Yao, Yu, Zhao, Shafran, Griffiths, Xiong, and Liu]{yao2023treeofthought}
Shunyu Yao, Dian Yu, Jeffrey Zhao, Izhak Shafran, Thomas~L. Griffiths, Caiming Xiong, and Jingjing Liu.
\newblock Tree of thoughts: Deliberate reasoning via structured chain of thought.
\newblock \emph{arXiv preprint arXiv:2305.10601}, 2023{\natexlab{b}}.
\newblock URL \url{https://arxiv.org/abs/2305.10601}.

\bibitem[Yao et~al.(2023{\natexlab{c}})Yao, Yu, Zhao, Shafran, Griffiths, Xiong, Narang, Galley, and Huang]{yao2023react}
Shunyu Yao, Dian Yu, Jeffrey Zhao, Izhak Shafran, Thomas~L. Griffiths, Caiming Xiong, Sharan Narang, Michel Galley, and Jiaji Huang.
\newblock React: Synergizing reasoning and acting in language models.
\newblock \emph{arXiv preprint arXiv:2210.03629}, 2023{\natexlab{c}}.
\newblock URL \url{https://arxiv.org/abs/2210.03629}.

\bibitem[Zheng et~al.(2023)Zheng, Wu, Musgrave, et~al.]{zheng2023swebench}
Olivia Zheng, Junda Wu, Katherine Musgrave, et~al.
\newblock Swe-bench: Can language models resolve real-world github issues?
\newblock \emph{arXiv preprint arXiv:2310.06770}, 2023.

\bibitem[Zhou et~al.(2022)Zhou, Sch{\"a}rli, Hou, Wei, Scales, Wang, Schuurmans, Cui, Bousquet, Le, and Chi]{zhou2022leasttomost}
Denny Zhou, Nathanael Sch{\"a}rli, Le~Hou, Jason Wei, Nathan Scales, Xuezhi Wang, Dale Schuurmans, Claire Cui, Olivier Bousquet, Quoc Le, and Ed H. Chi.
\newblock Least-to-most prompting enables complex reasoning in large language models.
\newblock \emph{arXiv preprint arXiv:2205.10625}, 2022.
\newblock URL \url{https://arxiv.org/abs/2205.10625}.

\end{thebibliography}
\bibliographystyle{iclr2025_conference}


\end{document}